\documentclass[runningheads]{llncs}
\usepackage{graphicx}
\usepackage[numbers]{natbib}
\usepackage[utf8]{inputenc} %
\usepackage[T1]{fontenc}    %
\usepackage{hyperref}       %
\usepackage{url}            %
\usepackage{booktabs}       %
\usepackage{amsfonts}       %
\usepackage{nicefrac}       %
\usepackage{microtype}      %
\usepackage{xcolor}         %

\usepackage{fontawesome}
\usepackage{amsmath}
\usepackage{bbm}
\usepackage{amssymb}
\usepackage{graphicx}
\usepackage{todonotes}
\usepackage{subcaption}
\usepackage[ruled,vlined]{algorithm2e}
\newenvironment{proofsketch}{%
  \proof}{\endproof}

\usepackage{cite}
\usepackage[misc]{ifsym}

\usepackage{geometry}
\begin{document}
\newcommand{\inside}[1]{\textsc{Int}(#1)}
\newcommand{\insidek}[2]{\textsc{Int}^{#1}(#2)}
\newcommand{\bincdf}[2]{U_{#1}(#2)}
\newcommand{\bincdfp}[3]{U_{#1,#2}(#3)}
\newcommand{\maxinsidef}{\hat{\Phi}}
\newcommand{\bhypercube}{\mathcal{H}}
\newcommand{\hamball}[1]{\mathcal{B}_{#1}}
\newcommand{\maxinside}{\Phi}
\newcommand{\outside}[1]{\partial #1}
\newcommand{\outsidek}[2]{\partial^{#1}#2}
\newcommand{\cil}[1]{\lceil #1 \rceil}
\newcommand{\flor}[1]{\lfloor #1 \rfloor}
\newcommand{\dham}[1]{|\outside{\hamball{#1}|}}
\newcommand{\prob}[1]{\textrm{Pr}(#1)}
\newcommand{\imspace}[2]{\mathcal{I}_{#1, #2}}
\newcommand{\imgspace}[3]{\mathcal{I}_{#1, #2, #3}}
\newcommand{\bitspace}[3]{\mathcal{B}_{#1, #2, #3}}
\newcommand{\pertspace}[3]{\hat{\mathcal{I}}_{#1, #2, #3}}
\newcommand{\appsec}{appendix}
\newcommand{\fspace}[2]{\mathcal{F}_{#1, #2}}
\newcommand{\mapperdef}{\mathcal{M}}
\newcommand{\elp}[1]{L^{#1}}
\newcommand{\normcdf}{\Phi}
\newcommand{\hamgraph}[2]{\mathcal{H}(#1, #2)}
\newcommand{\pertalg}{\textsc{FindPerturbation}}
\newcommand{\expansion}[1]{\textsc{Exp}(#1)}
\newcommand{\expk}[2]{\textsc{Exp}^{#1}(#2)}
\newcommand{\discreizedh}[2]{\mathcal{D}_{#2}(#1)}

\title{Image classifiers can not be made robust to small perturbations}
\author{Zheng Dai\inst{1}\orcidID{0000-0002-8828-1075} \and
David K. Gifford\inst{1}\orcidID{0000-0003-1709-4034}}
\authorrunning{Z. Dai et al.}
\institute{Computer Science and Artificial Intelligence Laboratory, Massachusetts Institute of Technology, Cambridge MA 02139, USA \\
\email{\{zhengdai,gifford\}@mit.edu}}
\maketitle
\begin{abstract}
The sensitivity of image classifiers to small perturbations in the input is often viewed as a defect of their construction. We demonstrate that this sensitivity is a fundamental property of classifiers. For any arbitrary classifier over the set of $n$-by-$n$ images, we show that for all but one class it is possible to change the classification of all but a tiny fraction of the images in that class with a perturbation of size $O(n^{1/\max{(p,1)}})$ when measured in any $p$-norm for $p \geq 0$. We then discuss how this phenomenon relates to human visual perception and the potential implications for the design considerations of computer vision systems.

\end{abstract}

\section{Introduction}

It has been observed that classifiers built on deep learning architectures are prone to misclassification given tiny perturbations on their inputs \citep{szegedy2013intriguing}. Because these perturbations are typically imperceptible, they are commonly thought of as adversarial \citep{papernot2016limitations}.  The existence of small perturbations that alter classifier decisions has motivated a plethora of research into how such perturbations may be engineered or prevented \citep{moosavi2016deepfool, madry2017towards, ma2018secure, tramer2020adaptive, machado2021adversarial}.
While adversarial perturbations are defined to be imperceptible to humans, the concept of imperceptibility is difficult to formally define. Therefore, the size of a perturbation is often implicitly adopted as a surrogate for perceptibility \citep{fawzi2018adversarial}.

Here we demonstrate that susceptibility to small perturbations is a fundamental property of any algorithm that partitions an image space into distinct classes. Specifically, we show that on any image space consisting of images with $n$-by-$n$ pixels and finite bit depth, there exists some universal constant $c$ (parametrized by the number of channels) such that most images in all but one class can have their classes changed with $cn$ pixel changes, a vanishingly small number compared to the $n^2$ pixels within the entire image for sufficiently large $n$. Similarly, we show that a perturbation with a $p$-norm of size $c'n^{1/p}$ suffices as well, for some $c'$ dependent on $p$, the number of channels, and the bit depth. Thus, the creation of a classifier that is robust to perturbations of the sizes described above is impossible.

Conversely, we also demonstrate that an upper bound on classifier robustness that applies universally to all image classifiers cannot be smaller than ours by more than a constant factor (parametrized by bit depth). Finally, we show how increasing the bit depth of our image space decreases classifier robustness under certain definitions of perturbation size.

Our bounds are unconditional, therefore they apply to classifiers based on human perception as well. We discuss the possible interpretations of this fact, and its potential implications for designing computer vision systems.

\subsection{Related work}

The sensitivity of neural networks to small perturbations was discovered in \citep{szegedy2013intriguing} where the authors remarked that perhaps adversarial examples form a dense low measure set analogous to the rationals. A serious effort to explain adversarial examples was undertaken in \citep{goodfellow2014explaining}, which suggests that adversarial examples is a consequence of high dimensional dot products between weights and inputs. However, their argument is not formal, and it has been shown that high dimensional dot products is neither necessary nor sufficient to explain adversarial images \citep{tanay2016boundary}.
Formal arguments bounding adversarial perturbations and robustness have been proven for specific instances \citep{gilmer2018adversarial, tsipras2018robustness}. However, the settings under which these theoretical results hold are usually highly idealized, and these arguments do not hold under more general settings.

The most general results for explaining adversarial examples comes from universal non-robustness bounds achieved through the use of isoperimetric bounds. This is the approach we take in this work. Isoperimetric results bound the surface area of any given volume in some space, so they are highly generalizable. The work presented in \citep{fawzi2018adversarial} uses an isoperimetric bound to bound the fraction of the space of natural images that is susceptible to changing classes under a small perturbation for any arbitrary classifier. However, they only consider perturbations measured by the Euclidean distance (2-norm), while our analysis encompasses perturbations measured by any $p$-norm. Furthermore, our bound is of a different nature as it considers the space of all images and is therefore unconditional and universal, while their bounds focus on image manifolds defined by generative functions and therefore are parametrized by the generator.

Isoperimetric bounds are also applied to understanding adversarial perturbations in \citep{diochnos2018adversarial}, where it is shown that for arbitrary classifiers over boolean inputs, most inputs can be pushed into a misclassification region with a small perturbation as long as the region occupies an asymptotically finite fraction of the input space. This work has since been extended to apply to a more general class of spaces in \citep{mahloujifar2019curse} using concentration bounds. Our work instead focuses on pushing images into different classification regions, rather than into a specific misclassification region, and is therefore of a slightly different nature. Also, unlike our analysis, their analysis does not preclude the existence of asymptotically infinitesimal classes of images that are robust to perturbations.

Our work also explores how these bounds apply to the human visual system due to their universality in contrast to prior work.

Studies attempting to understand adversarial perturbations in the human visual system usually do so by showing people adversarial images. This line of work has revealed that imperceptible adversarial perturbations may in fact be perceptible and influence human classifications \citep{elsayed2018adversarial, zhou2019humans}.
This line of work is very different from the work presented here: our approach is more theoretical, and our subsequent interpretations focus on perturbations that are clearly visible to humans despite being small.

In the remainder of this paper we provide a precise exposition of all our results as well as our terminology (Section 2), interpret these results (Section 3), and provide concluding remarks (Section 4). Proofs are mostly omitted and can be found in Appendix \ref{appdx-all-proofs}. 
\section{Results}

In this section we state universal non-robustness results for classifiers over images that can be encoded with finite bit strings. We then state how these non-robustness results are asymptotically the best we can achieve up to a constant factor, and we conclude by stating some results on how bit depth influences some of these bounds.

Intuitively, our results are a consequence of the high dimensional geometric phenomenon where measure concentrates near the boundary of sets in high dimensions.

\subsection{Preliminaries}

Images consist of pixels on a two dimensional grid, with each pixel consisting of a set of channels (for example R, G, and B) of varying intensity. Therefore, we define an \emph{$h$-channel image of size $n\times n$} to be a real valued tensor of shape $(n,n,h)$, where each entry is restricted to the interval $[0,1]$. The first two dimensions index the pixel, while the third indexes the channel. We use $\imgspace{n}{h}{\infty}$ to denote the set of all such images.

Only a finite subset of these images can be represented with a finite number of bits. Therefore, we define the set of all \emph{$h$-channel images of size $n\times n$ with bit depth $b$}, denoted $\bitspace{n}{h}{b}$, as the set of all bit valued tensors with shape $(n,n,h,b)$. The additional fourth dimension indexes the positions of a bit string that encodes the intensity of a channel. We map elements of $\bitspace{n}{h}{b}$ to $\imgspace{n}{h}{\infty}$ by mapping each length $b$ bit string to equally spaced values in $[0,1]$ with the largest value being 1 and the smallest being 0. We will use $\imgspace{n}{h}{b}$ to denote the image of $\bitspace{n}{h}{b}$ under this map. We will sometimes refer to $\imgspace{n}{h}{b}$ as \emph{discrete image spaces} to disambiguate them from $\imgspace{n}{h}{\infty}$, which we will refer to as the \emph{continuous image space}.

\subsubsection{Classifiers and Classes}

A classifier $\mathcal{C}$ is a function $\imgspace{n}{h}{b}\rightarrow \mathcal{Y}$, where $\mathcal{Y}$ is some finite set of labels. For each $y \in \mathcal{Y}$, we define the class of $y$ as the preimage of $y$, denoted as the set of images $\mathcal{C}^{-1}(y)$. We say that such a class is induced by $\mathcal{C}$. If a class takes up a large part of the image space, then it contains a lot of images that look like randomly sampled noise, since randomly sampling channel values from a uniform distribution yields a uniform distribution over the image space. Therefore, many images in these classes tend to be uninteresting, which motivates the following definition:

\begin{definition}
\label{def-interestingclass}
A class $C \subseteq \imgspace{n}{h}{b}$ is \emph{interesting} if it is not empty, and if it contains no more than half of the total number of images in $\imgspace{n}{h}{b}$.
\end{definition}

Note that if no class is empty, then no more than 1 class can be uninteresting. This is because classes are disjoint and so at most 1 class can contain more than half the total number of images.

\subsubsection{Perturbations and Robustness}

In order to discuss perturbations, we define addition and subtraction over tensors that are of the same shape to be element-wise, and we define the $p$-norm of a tensor $A$, denoted $\|A\|_p$, to be the $p$th root of the sum of the absolute values of the entries of $A$ raised to the $p$th power. $p$ is assumed to be a non-negative integer, and for the special case of $p=0$ we let $\|A\|_0$ be the number of non-zero entries in $A$.

We can then define what it means for an image to be robust to perturbations:

\begin{definition}
Let $\mathcal{C}:\imgspace{n}{h}{b}\rightarrow \mathcal{Y}$ be a classifier. We say an image $I \in \imgspace{n}{h}{b}$ is robust to $\elp{p}$-perturbations of size $d$ if for all $I' \in \imgspace{n}{h}{b}$, $\|I-I'\|_p \leq d$ implies $\mathcal{C}(I) = \mathcal{C}(I')$.

\end{definition}

We can then define what it means for a class to be robust to perturbations. Note that unless a class occupies the entire image space, it must contain some non-robust images, so the best we can hope for is to attain robustness for a large fraction of the images within a class. This is reflected in the following definition.

\begin{definition}
Let $\mathcal{C}:\imgspace{n}{h}{b}\rightarrow \mathcal{Y}$ be a classifier, and let $C$ be a class induced by it. Then we say that a class $C$ is $r$-robust to $\elp{p}$-perturbations of size $d$ if it is not empty, and the number of images $I \in C$ that are robust to $\elp{p}$-perturbations of size $d$ is at least $r|C|$, where $|C|$ is the number of images in $C$.

\end{definition}

\subsection{Universal upper bound on classifier robustness}

We can now state a universal non-robustness result that applies to all classifiers over discrete image spaces $\imgspace{n}{h}{b}$.

\begin{theorem}
\label{thm-main}
Let $\mathcal{C}: \imgspace{n}{h}{b} \rightarrow \mathcal{Y}$ be any classifier. Then for all real values $c > 0$, no interesting class is $2e^{-2c^2}$-robust to $\elp{p}$-perturbations of size $(2+c\sqrt{h}*n)^{1/\max(p,1)}$.

\end{theorem}

\begin{proofsketch}
We can use the images in $\imgspace{n}{h}{b}$ to form a graph where images are the vertices, and images are connected if and only if they differ at exactly one channel. In other words, the image tensors must differ at precisely one entry. Figure \ref{fig-proof1}a illustrates the construction of this graph. Note that graph distance between vertices coincides with the Hamming distance between the images represented by the vertices. Such graphs are known as Hamming graphs, and they have a vertex expansion (or isoperimetry) property ~\citep{harper1999isoperimetric} which implies that for any sufficiently small set, if we add all vertices that are within a graph distance of $\mathcal{O}(n)$ to that set, then the size of that set increases by at least some given factor (see Figure \ref{fig-proof1}b for an example).

We can then show that an interesting class $C$ cannot be too robust in the following way: suppose for contradiction that it is. Then there must be some set $C' \subseteq C$ that is pretty large, and has the property that all vertices within some graph distance of $C'$ are in $C$. We can then use the vertex expansion property to show that adding these vertices to $C'$ gives a set larger than $C$, which contradicts the assumption that all vertices within some graph distance $C'$. Plugging explicit values into this argument yields the statement of the theorem.

We can then generalize to $\elp{p}$-perturbations for arbitrary $p$ since each coordinate varies by at most 1 unit. The full proof can be found in Appendix \ref{appdx-proof-main}.\qed
\end{proofsketch}

\begin{figure}
  \begin{center}
        a)
        \includegraphics[width=0.6\textwidth]{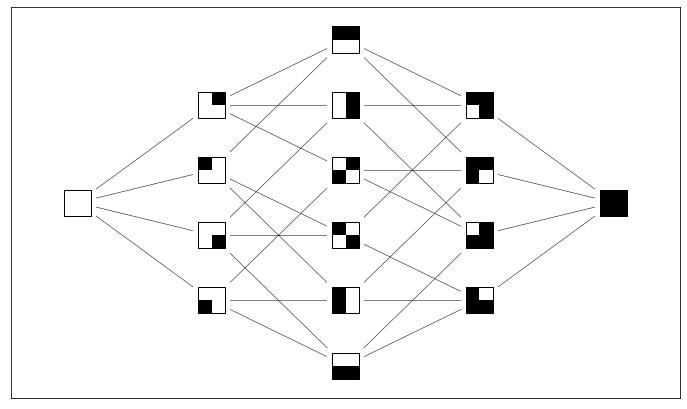}
        
        b)
        \includegraphics[width=0.6\textwidth]{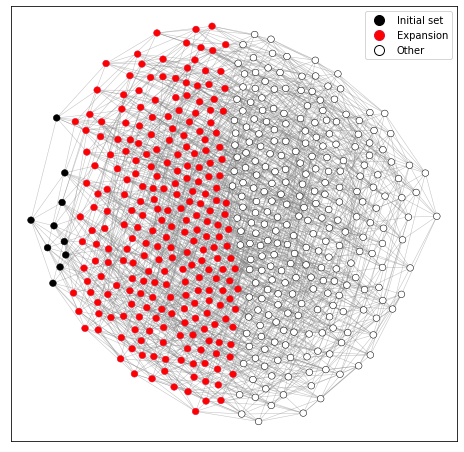}
  \end{center}
  \caption{Interpreting image spaces as Hamming graphs}
  \label{fig-proof1}
  a) We show how we construct a Hamming graph using the elements of $\imgspace{2}{1}{1}$, the space of binary images on four pixels. By construction, graph distance coincides exactly with Hamming distance.
  
  b) We demonstrate the expansion property of Hamming graphs on a Hamming graph constructed using $\imgspace{3}{1}{1}$ as the vertex set. If we pick some initial set of vertices (in black), then the set of vertices that are a graph distance of at most 3 (3 being $n$ in this case) from that initial set (in black and red) is much larger than that initial set. The nature of ``much larger'' is expanded on in Appendix \ref{appdx-proof-main}.
\end{figure}

Intuitively, the above results state that we can change the class of most ``interesting'' images with small perturbations that are on the order of $\mathcal{O}(n)$ pixel changes. The implications of this are considered in the discussion.

\subsubsection{The universal non-robustness results are asymptotically optimal up to a constant factor}

Up to a constant factor, the bounds in Theorem \ref{thm-main} are the best possible for a universal non-robustness result that applies to arbitrary predictors if we only consider $n$ and hold the number of channels per pixel $h$ and bit depth $b$ constant. In other words, there exists no bound on robustness that applies \emph{universally to all classifiers} that grows much more slowly in $n$ than the ones given in Theorem \ref{thm-main}. Therefore, if we wish to show that the classes induced by some classifier are not robust to, for instance, $\elp{0}$-perturbations of size $\mathcal{O}(log(n))$, more specific properties of that classifier would need to be considered.

To prove this, consider the classifier defined by Algorithm \ref{alg-allcord}.

\begin{algorithm}
\SetKwInOut{Input}{Input}
\Input{An image $I \in \imgspace{n}{h}{b}$}
\KwResult{A label belonging to $\{0,1\}$}
$S \gets 0$\;
\For{$x \gets 1$ \KwTo $n$}{
\For{$y \gets 1$ \KwTo $n$}{
\For{$a \gets 1$ \KwTo $h$}{
    $S \gets S+I_{x,y,a}$\;
}
}
}
\uIf{
$S < n^2h/2$
}{
\Return 0\;
}
\uElse{
\Return 1\;
}
\caption{Robust Classifier}
\label{alg-allcord}
\end{algorithm}

\begin{theorem}
\label{thm-robust}
Let $\mathcal{C}: \imgspace{n}{h}{b} \rightarrow \{0,1\}$ be the classifier described by Algorithm \ref{alg-allcord}. Then there exists an interesting class $C$ induced by $\mathcal{C}$ such that for all $c > 0$:

\begin{enumerate}
    \item $C$ is $(1-4c)$-robust to $\elp{p}$-perturbations of size $c\sqrt{h}*n - 2$ for all $p \leq 1$.
    \item $C$ is $(1-4c)$-robust to $\elp{p}$-perturbations of size $\frac{(c\sqrt{h}*n - 2)^{1/p}}{2^b-1}$ for all $p \geq 2$.
\end{enumerate}
\end{theorem}

\begin{proofsketch}
Given an image $I$, let $S(I)$ be the sum of all its channel values subtracted by $n^2h/2$. Then $I$ being robust to $\elp{1}$-perturbations of size $x$ is approximately equivalent to $S(I) \notin [-x, x]$. By the central limit theorem, the fraction of images $I$ such that $S(I) \notin [-cn\sqrt{h}, cn\sqrt{h}$] is some monotonic function of $c$ independent of $n$ and $h$ if $n^2h$ is sufficiently large, which is our desired result. Appendix \ref{appdx-proof-robust} provides a more careful analysis of this that does not rely on limiting behaviour and extends the result to all $p$-norms.\qed
\end{proofsketch}

Combining this statement with Theorem \ref{thm-main} then immediately yields the following statement, which implies that the statements in Theorem \ref{thm-main} are asymptotically optimal up to a constant factor:

\begin{corollary}
\label{cor-main}
For all integers $h,b \geq 1$, $p \geq 0$, and $r \in (0,1)$, there exist constants $c_1 \geq c_2 > 0$ and $n_0$ such that for any $n\geq n_0$ and labels $\mathcal{Y}$:

\begin{enumerate}
    \item No classifier $\mathcal{C}:\imgspace{n}{h}{b}\rightarrow \mathcal{Y}$ induces an interesting class that is $r$-robust to $\elp{p}$-perturbations of size $c_1n^{1/\max(p,1)}$.
    \item There exists a classifier $\mathcal{C}:\imgspace{n}{h}{b}\rightarrow \mathcal{Y}$ which induces an interesting class that is $r$-robust to $\elp{p}$-perturbations of size $c_2n^{1/\max(p,1)}$.
\end{enumerate}
\end{corollary}

We remark that the constant factor by which Theorem \ref{thm-main} misses optimality by is dependent on the bit depth $b$ for $p$-norms where $p \geq 2$, so significant improvements in the bound may still be possible when we consider it. We make some progress towards this in Theorem \ref{thm-cont-discr}.

\subsubsection{Classifier robustness to \texorpdfstring{$\elp{2}$}{Lg}-perturbations decreases with increasing bit depth}

In this section we investigate the role played by the bit depth $b$. Theorem \ref{thm-robust} has a dependency on $b$ when considering $\elp{p}$-perturbations for $p \geq 2$. Is there an alternative construction by which we can remove any such dependency altogether to close the gap between Theorem \ref{thm-main} and \ref{thm-robust}? We demonstrate in this section that we cannot.

Specifically, we can derive a universal upper bound on robustness that is dependent on $b$, such that as $b$ grows without bound, this bound approaches some constant independent of the number of pixels in the image.

\begin{theorem}
\label{thm-cont-discr}
Let $\mathcal{C}: \imgspace{n}{h}{b} \rightarrow \mathcal{Y}$ be any classifier. Then for all real values $c > 0$ and $p \geq 2$, no interesting class is  $2e^{-c^2/2}$-robust to $\elp{p}$-perturbations of size $\big( c + 2\frac{n\sqrt{h}}{2^{b}} \big)^{2/p}$.
\end{theorem}

\begin{proofsketch}
We will focus on the 2-norm. Extension to higher $p$-norms is straightforward and is given as part of the full proof found in Appendix \ref{appdx-cont-discr-proof}.
The main idea of the proof rests on the fact that if we extend the classifier to the continuous image space with something like a nearest neighbour approach, the measure of the images that are robust to perturbations of a constant size is small (the statement and proof may be found in Appendix \ref{appdx-continuous-results}). Therefore, if we randomly jump from an image in the discrete image space to an image in the continuous image space, with high probability we will be within a constant distance of an image of a different class. The size of this random jump can be controlled with a factor that shrinks with increasing bit depth. Summing up the budget required for this jump, the perturbation required on the continuous image space, and the jump back to the discrete image space yields the desired bound.\qed
\end{proofsketch}

We remark that this suggests that the bounds in Theorem \ref{thm-main} pertaining to $\elp{p}$-perturbations for $p \geq 2$ can be improved to reflect its dependency on the bit depth $b$. However, additional work would need to be done to show that the component that shrinks with $b$ scales with $n^{1/p}$ rather than $n^{2/p}$.

\subsection{Summary of bounds and their relation to average image distances}

We conclude this section by recapitulating the bounds we derived and compare them to the average distances between images for context.

We summarize the bounds we derived in Table \ref{table-summary}. For parsimony, we have reparametrized the bounds in terms of the robustness $r$ in Table \ref{table-summary}, although the equations look more complex as a result. In terms of image size $n$, the bounds stated for the $0$-norm and $1$-norm are asymptotically optimal up to a constant factor. The bounds for the other $p$-norms are also asymptotically optimal up to a constant factor, although the constant is parametrized by the bit depth $b$. We showed that the presence of $b$ in our lower bound is not an artifact of our construction: robustness really does drop as $b$ increases (Theorem \ref{thm-cont-discr}).

\begin{table*}[t]
\caption{Bounds for attainable robustness. Rather than leaving the robustness and bound parametrized by a separate constant $c$, the bounds have been reparametrized in terms of the robustness $r$. The upper bound should be understood as ``no classifier induces an interesting class that is $r$-robust to perturbations of these sizes'' and the lower bound should be interpreted as ``there exists a classifier that induces an interesting class that is $r$-robust to perturbations of these sizes''.}
\label{table-summary}
\vskip 0.15in
\begin{center}
\begin{small}
\begin{sc}
\begin{tabular}{lll}
\toprule
Perturbation & Upper bound & Lower bound \\
\midrule
\shortstack{$\elp{0}$-perturbation\\
$\elp{1}$-perturbation}

& $2 + \sqrt{\dfrac{h}{2}ln(\dfrac{2}{r})} * n$
& $-2 + \big( \dfrac{1-r}{4} \big) \sqrt{h}*n$
\\ %
\\ %
\shortstack{$\elp{p}$-perturbation,\\ $p \geq 2$}
& \shortstack{$\min\bigg(
\big(
2 + \sqrt{\dfrac{h}{2}ln(\dfrac{2}{r})} * n
\big) ^ {1/p}
,%
\big(
\sqrt{2 ln( \dfrac{2}{r} )} + \dfrac{2\sqrt{h}}{2^{b}}*n
\big) ^ {2/p}
\bigg)$
}
& $\dfrac{\bigg(
-2 + \big( \dfrac{1-r}{4} \big) \sqrt{h}*n
\bigg)^{1/p}
}
{2^b - 1}$
\\
\bottomrule
\end{tabular}
\end{sc}
\end{small}
\end{center}
\vskip -0.1in
\end{table*}

To conclude this section, we contextualize the bounds derived in this section by comparing them to typical distances between random elements of the image space. We can show that for a pair of images $I, I' \in \imgspace{n}{h}{b}$ that are sampled independently and uniformly, we have:

\begin{align}
    \mathbb{E}[\|I-I'\|_p] \geq k_{h,b,p}n^{2/\max(1,p)}
\end{align}

Where $k_{h,b,p}$ is some constant parametrized by $h$, $b$, and $p$. See Appendix \ref{appdx-avg-distance} for additional details.

Combining this with Corollary \ref{cor-main} shows that if $n$ is sufficiently large, for 99\% (or some arbitrarily high percentage) of images $I''$ within some interesting class, there exists some $c_{h,b,p}$ parametrized by $h$, $b$, and $p$ such that:

\begin{align}
\dfrac{\min_{X \in \imgspace{n}{h}{b}, \mathcal{C}(I'') \neq \mathcal{C}(X)} \|I''-X\|_p }{ \mathbb{E}[\|I-I'\|_p] }
\leq
    c_{h,b,p}n^{-\frac{1}{\max(p,1)}}
\end{align}

The right hand side approaches $0$ as $n$ grows without bound, so compared to typical distances one finds in an image space, the distance of an image to an image outside of its class is vanishingly small in any $p$-norm it is measured in.
\section{Human classification decisions are subject to universal bounds on robustness}

Since the bounds in Table \ref{table-summary} apply universally to any image classifier, they must also apply to the human visual system. Although there are many nuances to consider when interpreting the human visual system as a classifier, we can abstract most of them out by considering the following system for classifying images: we imagine a room containing a person and a monitor that displays images of size $n$-by-$n$. The person then has access to a selection of labels to label images with. To classify an image, the image is first fed into a memoization subroutine that checks if the image has been seen before and returns the label it was previously labelled with if it has. If the image has not been seen before, it is then displayed on the monitor, and the person is allowed to select a single label (or no label at all) to apply to the image. We remark that this classifier can be \emph{concretely realized}, so we cannot dismiss it as simply an abstract construction.

This system acts as a classifier which partitions the set of all images into disjoint classes, therefore the bounds in Table \ref{table-summary} must apply. To simplify the discussion, we make an assumption about the human based classifier: at least half the images in the image space are unlabelled. This condition is met if there is no label applicable to images that look like random static. Intuitively, we can interpret labelled images as ones that are ``meaningful'' and unlabelled images as ones that are ``meaningless'' if the label set is sufficiently large.

If the unlabelled images occupy at least half the image space, then the labelled images form an interesting class (as defined by Definition \ref{def-interestingclass}). Therefore, the bounds in Table \ref{table-summary} apply, which means that a large fraction of labelled images can be turned into unlabelled images with a small perturbation.

If we return to the intuition that labels formalize the notion of ``meaning'', this means that for most ``meaningful'' images, the meaning can be erased with only a tiny fluctuation. Conversely, the ``meaning'' present in most ``meaningful'' images arises from tiny fluctuations.

The bounds in Table \ref{table-summary} then state that such ``meaning'' can fit in a perturbation of size $\mathcal{O}(n)$ when measured using the 1-norm or via the Hamming distance. This can be interpreted as a statement about the saliency of line drawings. Figure \ref{fig-crossWrite} gives a demonstration of how line drawings are small perturbations that contain ``meaning''.

\begin{figure}
  \begin{center}
        a) \includegraphics[width=0.25\textwidth]{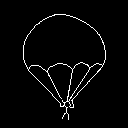}
        b) \includegraphics[width=0.25\textwidth]{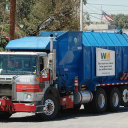}
        c) \includegraphics[width=0.25\textwidth]{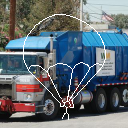}
        
        d) \includegraphics[width=0.25\textwidth]{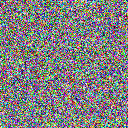}
        e) \includegraphics[width=0.25\textwidth]{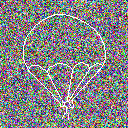}
  \end{center}
  \caption{Small perturbations form meaningful patterns}
  \label{fig-crossWrite}
  We show how a small perturbation (a) when overlaid on either a (b) natural image (sourced from~\citep{imagenette}) or a (d) uniformly randomly drawn image is able to add meaningful information, in this case a parachute, to those images (c, e).
\end{figure}

Table \ref{table-summary} also states that when we raise the bit depth of the image space to be arbitrarily high, ``meaning'' can fit in a perturbation of size $\mathcal{O}(1)$ when measured using a $p$-norm with $p \geq 2$. Line drawings do not necessarily fulfill this criterion, so the interpretation of this fact is more difficult. The human visual system is known to be particularly sensitive to certain small cues \citep{liu2014seeing}, but a unified understanding remains elusive.

Understanding the nature of the small perturbations that humans are sensitive to is not merely of academic curiosity.
The results summarized in Table \ref{table-summary} show that no computer vision system can be robust to small perturbations. However, a computer vision system that is aligned to the human visual system \emph{ought} not be robust to small perturbations, since the human visual system is not robust either. Over the past decade we have learned that standard machine learning methodology does not automatically produce vision systems that are aligned to the human visual system with respect to small perturbations ~\citep{szegedy2013intriguing}, and methodologies that seek to produce such vision systems still contain misalignments ~\citep{tramer2020adaptive}. A deeper understanding of how small perturbations affect the human visual system may inform the development of such methodologies (for example we may wish to explicitly train computer vision systems on human sensitive small perturbations), which is becoming increasingly necessary as computer vision systems become increasingly deployed in safety and security critical applications, where the trustworthiness of the system is essential~\citep{pereira2020challenges, ma2018secure}.

\section{Conclusion}

We have derived universal non-robustness bounds that apply to any arbitrary image classifier. We have further demonstrated that up to a constant factor, these are the best bounds attainable. These bounds reveal that most images in any interesting class can have their class changed with a perturbation that is asymptotically infinitesimal when compared to the average distance between images.

We then discuss how these universal properties of classifiers relate to the human visual system. We show that part of our results can be interpreted as the sensitivity of the human visual system to line drawings, which are tiny signals when measured using the 1-norm or 0-norm. However, line drawings can still be ``large'' when measured using the 2-norm, so a full understanding remains the subject of future work.

Our results focuses on image \emph{classifiers}, which make hard decisions when labelling images. However, vision models underlying the classifiers can make soft decisions, which are then further processed into hard decisions. The applicability of our results to such underlying vision models will be the subject of future investigation.
\bibliographystyle{plainnat}
\bibliography{references.bib}
\newpage
\appendix
\section{Proofs of Statements}
\label{appdx-all-proofs}
\subsection{Proof of Theorem \ref{thm-main}}
\label{appdx-proof-main}

\subsubsection{Properties of binomial coefficients}
We will work with binomial coefficients extensively. To simplify some of our statements, we will extend the definition of a binomial coefficient to work with any $n > 0$ and arbitrary integer $k$:

\begin{align}
    \binom{n}{k} = \begin{cases}
\dfrac{n!}{k!(n-k)!} & \text{if $0 \leq k \leq n$}\\
0 & \text{otherwise}
\end{cases}
\end{align}

Binomial coefficients can be bound in the following way:

\begin{lemma}
$\binom{n}{k} < \dfrac{2^n}{\sqrt{n}}$ when $n \geq 1$.
\label{lemma-binomial-mode}
\end{lemma}

\begin{proof}
We first note that $n!$ is bounded by the following for all $n \geq 1$~\citep{robbins1955remark}:
\begin{equation}
    \sqrt{n}\dfrac{n^n}{e^n} < \dfrac{n!}{\sqrt{2\pi}} < \sqrt{n}\dfrac{n^n}{e^n}e^{1/(12n)}
\end{equation}
Applying the appropriate inequalities for the numerator and denominator yields the following for when $n$ is even:
\begin{align}
\label{eqn-helper-17}
    \binom{n}{k} \leq \binom{n}{n/2} = \dfrac{n!}{((n/2)!)^2}
    < 2\dfrac{2^n}{\sqrt{n}}\dfrac{e^{1/(12n)}}{\sqrt{2\pi}}
\end{align}
When $n$ is odd, we have:
\begin{align}
    \binom{n}{k} &\leq \binom{n}{\flor{n/2}} \\
    &= \dfrac{1}{2}\binom{n+1}{(n+1)/2} \\
    &< 2\dfrac{2^n}{\sqrt{n+1}}\dfrac{e^{1/(12(n+1))}}{\sqrt{2\pi}}\\
    &< 2\dfrac{2^n}{\sqrt{n}}\dfrac{e^{1/(12n)}}{\sqrt{2\pi}}
\end{align}
Where the third comparison is an application of Equation \ref{eqn-helper-17}.

If $n \geq 1$, we have $\frac{e^{1/(12n)}}{\sqrt{2\pi}} < 0.5$, which proves the claim.
\qed
\end{proof}

It will also be useful to define the following cumulative sums (which are also the tails of binomial distributions):

\begin{align}
    \bincdfp{n}{p}{k} &= \begin{cases}
\sum_{i=0}^k \binom{n}{i}p^k(1-p)^{n-k} & \text{if $k \geq 0$}\\
0 & \text{otherwise}
\end{cases}
\end{align}

We can show that the ratio of these cumulative sums are monotonic increasing:

\begin{lemma}
\label{lemma-monotonic-ratio}
Let $p \in (0,1)$. Then $\frac{ \bincdfp{n}{p}{x-k} }{ \bincdfp{n}{p}{x} }$ is monotonic increasing in $x$, where $0 \leq x \leq n$ and $k$ is any positive integer.
\end{lemma}

\begin{proof}
First, we note that the ratio $\binom{n}{x-k}/\binom{n}{x}$ is monotonic increasing in $x$ when $x \geq 0$. This holds by definition if $x-k < 0$. Otherwise, we have the following:

\begin{align}
    \dfrac{ \binom{n}{x-k}/\binom{n}{x} }{ \binom{n}{x-k+1}/\binom{n}{x+1} } = 
    \dfrac{ (n-x) }{ (n-x+k) } *
    \dfrac{ (x-k+1) }{ (x+1) }
    \leq 1
\end{align}

We then claim the following holds for all $x$ where $0 \leq x \leq n-1$:

\begin{align}
    \dfrac{ \bincdfp{n}{p}{x-k} }{ \bincdfp{n}{p}{x} } \leq
    \dfrac{ \bincdfp{n}{p}{x-k+1} }{ \bincdfp{n}{p}{x+1} } \leq
    \dfrac{ \binom{n}{x-k+1}(1-p)^{k} }{ \binom{n}{x+1}p^{k} }
\end{align}

The above holds with equality when $x-k+1 < 0$. If $x-k+1 =0$, the above also holds: the leftmost ratio is 0. For the other two ratios, if we multiply the rightmost ratio by $(1-p)^{n-k}$ above we can see that the numerators are equal while the denominator of the rightmost ratio is smaller. Otherwise, by induction on $x$ we have:

\begin{align}
    \dfrac{ \bincdfp{n}{p}{x-k} }{ \bincdfp{n}{p}{x} }
    &\leq \dfrac{ \binom{n}{x-k}(1-p)^k }{ \binom{n}{x}p^k }\\
    &\leq \dfrac{ \binom{n}{x-k+1}(1-p)^k }{ \binom{n}{x+1}p^k }\\
    &= \dfrac{ \binom{n}{x-k+1}p^{x-k+1}(1-p)^{n-x+k-1} }{ \binom{n}{x+1}p^{x+1}(1-p)^{n-x+1} }
\end{align}

Where the first inequality follows by induction, and the second inequality follows because $\binom{n}{x-k}/\binom{n}{x}$ is monotonic increasing in $x$.

For any positive numbers $a$, $c$ and strictly positive numbers $b$, $d$, where $\frac{a}{b} \leq \frac{c}{d}$, we have $\frac{a}{b} \leq \frac{a+c}{b+d} \leq \frac{c}{d}$ because:

\begin{align}
    \dfrac{d}{d\lambda}
    \bigg(
    \dfrac{a+\lambda c}{b +\lambda d}
    \bigg)
    = \dfrac{bc -ad}{ (b + \lambda d)^2 } \geq 0
\end{align}

Therefore, we have:

\begin{align}\nonumber
    &\dfrac{ \bincdfp{n}{p}{x-k} }{ \bincdfp{n}{p}{x} }\\
    &\leq \dfrac{ \bincdf{n}{x-k} + \binom{n}{x-k+1}p^{x-k+1}(1-p)^{n-x+k-1} }
    { \bincdfp{n}{p}{x} + \binom{n}{x+1}p^{x+1}(1-p)^{n-x+1} }\\
    &= \dfrac{ \bincdfp{n}{p}{x-k+1} }{ \bincdfp{n}{p}{x+1} }\\
    &\leq \dfrac{ \binom{n}{x-k+1}(1-p)^k }{ \binom{n}{x+1}p^k }
\end{align}

As claimed. Carrying on the induction up to $x=n-1$ yields the statement.\qed
\end{proof}

\subsubsection{Bounding the interior of a set over a Hamming graph}
We will prove our main results by an application of isoperimetry bounds over a Hamming graph. Let $Q$ be a set of $q$ symbols. Then we define the $n$ dimensional Hamming graph over $q$ letters, denoted $\hamgraph{n}{q}$, as the graph with a vertex set $Q^n$ and an edge set containing all edges between vertices that differ at precisely one coordinate. For example, $\hamgraph{n}{2}$ is isomorphic to the Boolean hypercube. We will use $V(\hamgraph{n}{q})$ to denote the vertex set of the Hamming graph.

Let $S \subseteq \hamgraph{n}{q}$. We define the expansion of $S$, denoted $\expansion{S}$, as the set of vertices that are either in $S$ or have a neighbour in $S$. Since $\expansion{.}$ inputs and outputs sets of vertices, we can iterate it. We will use $\expk{k}{.}$ to denote $k$ applications of $\expansion{.}$.

We now adapt a a result from \citep{harper1999isoperimetric} (Theorem 3 in the paper).

\begin{lemma}[Isoperimetric Theorem on Hamming graphs]
\label{lemma-iso-hamgraph}
Let $S \subsetneq \hamgraph{n}{q}$. Then:
\begin{align}\nonumber
    \dfrac{|\expk{k}{S}|}{|V(\hamgraph{n}{q})|} \geq \min\{&
    \bincdfp{n}{p}{r+k} \\ \nonumber
    |& \bincdfp{n}{p}{r} = \dfrac{|S|}{|V(\hamgraph{n}{q})|},\\ 
    & p \in (0,1),r\in[0,n-k)
    \}
\end{align}
\end{lemma}

To work with this we first obtain bounds for the expression on the right hand side of Lemma \ref{lemma-iso-hamgraph}.

\begin{lemma}
\label{lemma-bound-any-binomial}
Let $p$ be any value in $(0,1)$. Let $n > r \geq k$ such that $\bincdfp{n}{p}{r} \leq \frac{1}{2}$. Then $\frac{\bincdfp{n}{p}{r-k}}{\bincdfp{n}{p}{r}} \leq 2e^{-2(k-1)^2/n}$.
\end{lemma}

\begin{proof}
Let $X$ be a binomially distributed random variable with $n$ trials and probability of success $p$. Let $r$ be the median of $X$. We have $r \leq np+1$ because the median and mean differ by at most 1 ~\citep{kaas1980mean}.

$\bincdfp{n}{p}{r-k}$ can be interpreted as $\prob{X \leq r-k}$, We can then apply Hoeffding's inequality~\citep{hoeffding1994probability}:

\begin{align}
    \prob{X \leq r-k} &= \prob{X \leq np+1-k}\\
    &\leq e^{-2(k-1)^2/n}
\end{align}

Since $r$ is the median of $X$, we also have $\bincdfp{n}{p}{r} \geq \frac{1}{2}$. Combining this with the above equation gives:

\begin{align}
    \frac{\bincdfp{n}{p}{r-k}}{\bincdfp{n}{p}{r}} \leq 2e^{-2(k-1)^2/n}
\end{align}

Since $\frac{\bincdfp{n}{p}{x-k}}{\bincdfp{n}{p}{x}}$ is monotonically increasing via Lemma \ref{lemma-monotonic-ratio}, this also implies that the above relation holds for all smaller $r$. This completes the proof.\qed
\end{proof}

We can then plug this into Lemma \ref{lemma-iso-hamgraph} to obtain a non-robustness result on Hamming graphs, which we will then apply to image spaces. %

\begin{theorem}
\label{thm-hamgraph-bound}
Let $S \subsetneq V(\hamgraph{n}{q})$ such that $|S| \leq |V(\hamgraph{n}{q})|/2$, and $c > 0$ be any number. Let $S' \subseteq S$ be the set of vertices for which no path with $c\sqrt{n}+2$ edges or less leads to a vertex not in $S$. Then $\frac{|S'|}{|S|} < 2e^{-2c^2}$.
\end{theorem}

\begin{proof}
Suppose for contradiction that $|S'| \geq 2e^{-2c^2}|S|$. Since for any vertex in $S'$ no path with $c\sqrt{n}+2$ edges or less leads to a vertex outside of $S$, we have $\expk{c\sqrt{n}+2}{S'} \subseteq S$. Then:

\begin{align}\nonumber
    |\expk{c\sqrt{n}+2}{S'}| \geq& |V(\hamgraph{n}{q})|\min\{
    \bincdfp{n}{p}{r+c\sqrt{n}+2} \\ \nonumber
    &| \bincdfp{n}{p}{r} = \dfrac{|S'|}{|V(\hamgraph{n}{q})|},\\
    & p \in (0,1),r\in[0,n-c\sqrt{n}-2)
    \}\\
    \geq& 2e^{2(c\sqrt{n} + 1)^2/n} |S'|\\
    >& 2e^{2c^2} |S'|
\end{align}

The first relation follows from Lemma \ref{lemma-iso-hamgraph} and the second follows from Lemma \ref{lemma-bound-any-binomial}. Lemma \ref{lemma-bound-any-binomial} applies since $\expk{c\sqrt{n}+2}{S'} \subseteq S$, so $|\expk{c\sqrt{n}+2}{S'}| \leq |S| \leq \frac{1}{2}$.

But then $|\expk{c\sqrt{n}+2}{S'}| > 2e^{-2c^2} |S'| \geq |S|$, which implies that $\expk{c\sqrt{n}+2}{S'} \nsubseteq S$. This is a contradiction, so we obtain our desired statement.\qed
\end{proof}
\subsubsection{Proving Theorem \ref{thm-main}}

Let $\mathcal{C}: \imgspace{n}{h}{b} \rightarrow \mathcal{Y}$ be a classifier and let $C \subseteq \imgspace{n}{h}{b}$ be any interesting class induced by $\mathcal{C}$.

\begin{lemma}
\label{lemma-l0-discr-main}
$C$ is not $2e^{-2c^2}$-robust to $\elp{0}$-perturbations of size $c\sqrt{h}*n+2$.
\end{lemma}

\begin{proof}
Let $\mapperdef: V(\hamgraph{n^2h}{2^b}) \rightarrow \imgspace{n}{h}{b}$ be the following bijection: first let $Q$ be a set of $2^n$ equally spaced values between 0 and 1, where the largest value is 0 and the smallest is 1. Then the elements of $V(\hamgraph{n^2h}{2^b})$ can be viewed as $Q^{n^2h}$. We then map elements from $Q^{n^2h}$ to $\imgspace{n}{h}{b}$ such that the inverse operation is a flattening of the image tensor. Note that such a mapping preserves graph distance on $V(\hamgraph{n^2h}{2^b})$ as Hamming distance on $\imgspace{n}{h}{b}$.

Let $C' \subseteq C$ be the set of images that are robust to $\elp{0}$-perturbations of size $c\sqrt{h}*n+2$.
Let $S = \mapperdef^{-1}(C)$ and $S' = \mapperdef^{-1}(C')$. $S'$ is then the set of vertices for which no path with $c\sqrt{h}*n+2$ edges or less leads to a vertex outside of $S$.

$C$ is an interesting class and $\mapperdef(.)$ preserves cardinality due to it being a bijection. Therefore $|C'| \leq |V(\hamgraph{n^2h}{2^b})|/2$, so by Theorem \ref{thm-hamgraph-bound} we have $|S'|/|S| < 2e^{-2c^2}$. Again, since $\mapperdef(.)$ preserves cardinality, this implies that $|C'|/|C| < 2e^{-2c^2}$, which means that $C$ is not $2e^{-2c^2}$-robust to $\elp{0}$-perturbations of size $c\sqrt{h}*n+2$.\qed
\end{proof}

We remark that if the domain of $\mapperdef(.)$ is changed to $\hamgraph{n^2}{h2^b}$, the above argument also shows that $C$ is not $2e^{-2c^2}$-robust to $cn+2$ pixel changes.

It is straightforward to generalize this to $p$-norms with larger $p$.

\begin{lemma}
\label{lemma-thm1-higherp}
$C$ is not $2e^{-2c^2}$-robust to $\elp{p}$-perturbations of size $(c\sqrt{h}*n+2)^{1/p}$.
\end{lemma}

\begin{proof}
Let $S_1$ be the set of images that are $r$-robust to $\elp{0}$-perturbations of size $d$, and let $S_2$ be the set of images that are $r$-robust to $\elp{p}$-perturbations of size $d^{1/p}$.

Suppose $I \notin S_1$. Then there exists some image $I'$ in a different class from $I$ such that $\|I-I'\|_{0} \leq d$. Therefore, for all $p > 0$, we have:

\begin{align}
    d &\geq \|I-I'\|_{0}\\
    &= \sum_{x,y,c} \cil{ |I_{x,y,c} - I'_{x,y,c}| }\\
    &\geq \sum_{x,y,c} |I_{x,y,c} - I'_{x,y,c}|^{p}\\
    &= (\|I-I'\|_{p})^{p}
\end{align}

Where the second and third relation follows from the fact that channel values are contained in $[0,1]$. Therefore, $I \notin S_2$ either since $\|I-I'\|_{p} \leq d^{1/p}$. Taking the contraposition yields $S_2 \subseteq S_1$.

Setting $d = c\sqrt{h}*n+2$ and applying Lemma \ref{lemma-l0-discr-main} gives the desired result.\qed
\end{proof}
\subsection{Proof of Theorem \ref{thm-robust}}
\label{appdx-proof-robust}

\subsubsection{Anti-concentration inequalities}

We first prove an anti-concentration lemma concerning the binomial distribution.

\begin{lemma}
\label{lemma-binomial-spread}
Let $X$ be a random variable following the binomial distribution with $n$ trials and a probability of success of 0.5. Let $Y$ be a discrete random variable independent of $X$ whose distribution is symmetric about the origin. Then for any $t$ where $t < \mathbb{E}[X]$ and $t - \flor{t} = 1/2$, we have:

\begin{align}
    \prob{X+Y \leq t} \geq \prob{X < t}
\end{align}
\end{lemma}

\begin{proof}
We have the following:
\begin{align}
    \prob{X+Y \leq t} =& \prob{X+Y \leq t, X < t} \\
    \nonumber &+ \prob{X+Y \leq t, X > t}\\
    \prob{X < t} =& \prob{X+Y \leq t, X < t} \\
    \nonumber &+ \prob{X+Y > t, X < t}
\end{align}
Therefore it suffices to show that $\prob{X+Y \leq t, X > t} \geq \prob{X+Y > t, X < t}$. We have for any $r \geq 0$:
\begin{align}
\label{eqn-helper-43}
    \prob{X+Y \leq t, X = t+r}
    &= \prob{Y \leq -r}\prob{X = t+r}\\
\label{eqn-helper-44}
    &\geq \prob{Y > r}\prob{X = t+r}\\
\label{eqn-helper-45}
    &\geq \prob{Y > r}\prob{X = t-r}\\
    &= \prob{X+Y > t, X = t-r}
\end{align}
Where Equation \ref{eqn-helper-43} follows from the independence of $X$ and $Y$, Equation \ref{eqn-helper-44} follows from the symmetry of the distribution of $Y$, and Equation \ref{eqn-helper-45} follows from our assumption that $t < \mathbb{E}[X]$ and $t - \flor{t} = 1/2$.

Summing over all positive $r$ for which $\prob{X = t+r} \geq 0$ yields the desired result.\qed
\end{proof}

\begin{lemma}
\label{lemma-anti-concentration}
Let $X_1, X_2, ..., X_n$ be independently and identically distributed random variables such that each $X_i$ is uniformly distributed on $2k$ evenly spaced real numbers $a=r_1 < r_2 < ... < r_{2k}=b$. Then for $t > 0$, we have:

\begin{align}
    \prob{\sum_{i=1}^n X_i \leq (\sum_{i=1}^n\mathbb{E}[X_i]) -t+(b-a)} > \dfrac{1}{2} - \dfrac{2t}{\sqrt{n}(b-a)}
\end{align}
\end{lemma}

\begin{proof}
Let $Y_1, Y_2, ... ,Y_n$ be independently and identically distributed Bernoulli random variables with $p=0.5$. Let $Z_1, Z_2,...,Z_n$ be a set of independently and identically distributed random variables uniformly distributed between the integers between $1$ and $k$ inclusive. If the $Y$s and $Z$s are independent of each other as well, we have:

\begin{align}
    \sum_{i=1}^n (X_i - \mathbb{E}[X_i])
    =& \dfrac{b-a}{2k-1} \sum_{i=1}^n (kY_i+Z_i - \mathbb{E}[kY_i+Z_i])\\
    \nonumber
    =& k\dfrac{b-a}{2k-1}
    \big(
    (\sum_{i=1}^n Y_i)
    + (\sum_{i=1}^n \dfrac{Z_i - \mathbb{E}[Z_i]}{k})\\
    &- (\sum_{i=1}^n \mathbb{E}[Y_i])
    \big)
\end{align}

Let $\sum_{i=1}^n Y_i = B$, $\sum_{i=1}^n \dfrac{Z_i - \mathbb{E}[Z_i]}{k} = D$, and $k\frac{b-a}{2k-1} = c$. Then for any $t > 0$, we have:

\begin{align}
    \prob{ \sum_{i=1}^n (X_i - \mathbb{E}[X_i]) \leq -t }
    =& \prob{ B+D \leq -\dfrac{t}{c} + \mathbb{E}[B] }\\
    \nonumber
    \geq& \prob{ B+D\\
    &\leq -\dfrac{t}{c} + \mathbb{E}[B] - u }\\
    \label{eqn-helper-52}
    \geq& \prob{ B < -\dfrac{t}{c} + \mathbb{E}[B] - 1 }\\
    \nonumber
    \geq& \prob{B-\mathbb{E}[B] \\
    &< -\dfrac{2t}{b-a} - 1}\\
    \nonumber
    \geq& \dfrac{1}{2} - \prob{B-\mathbb{E}[B]\\
    &\in [-\dfrac{2t}{b-a} - 1, 0]}\\
    \nonumber
    \geq& \dfrac{1}{2} - \binom{n}{\flor{n/2}}2^{-n}\\ 
    &(\dfrac{2t}{b-a} + 2)
    \label{eqn-helper-56}
\end{align}

Where $1 \geq u \geq 0$ is chosen such that $-\frac{t}{c} + \mathbb{E}[B] - u$ is the average of two adjacent integers. Equation \ref{eqn-helper-52} is then an application of Lemma \ref{lemma-binomial-spread} since $B$ is binomially distributed with $p=0.5$ and $D$ has a distribution that is symmetric about the origin, and Equation \ref{eqn-helper-56} follows from the fact that no more than $x+1$ values are supported on an interval of length $x$, and no supported value has probability greater than $\binom{n}{\flor{n/2}}2^{-n}$.

Observing that $\binom{n}{\flor{n/2}}2^{-n} < \frac{1}{\sqrt{n}}$ due to Lemma \ref{lemma-binomial-mode} and substituting $t$ with $t-(b-a)$ yields the desired result.\qed
\end{proof}

\subsubsection{Proving Theorem \ref{thm-robust}}
Let $A: \imgspace{n}{h}{b} \rightarrow \{0,1\}$ be described by Algorithm \ref{alg-allcord}. In other words, it is the classifier that inputs an image, sums all of its channels, and outputs $0$ if the sum is less than $n^2h/2$ and $1$ otherwise. Let $Z$ be the class of images that $A$ outputs $0$ on. Note that $Z$ is an interesting class since it cannot be larger than its complement, so it suffices to prove that $Z$ is robust.

\begin{lemma}
\label{lemma-robust-l1}
$Z$ is $(1-4c)$-robust to $\elp{1}$-perturbations of size $c\sqrt{h}*n - 2$
\end{lemma}

\begin{proof}
Let $Z' \subseteq Z$ be the set of images in $Z$ that are robust to $\elp{1}$-perturbations of size $c\sqrt{h}*n - 2$.
Let $I$ be a random image sampled uniformly. Then $|Z'| = \prob{I \in Z'}2^{-(n^2hb)}$.
We then have the following:
\begin{align}
\prob{I \in Z'}
&=
\prob{ \sum_{x,y,a}I_{x,y,a} + c\sqrt{h}*n - 2 < n^2h/2 }\\
&\geq
\prob{ \sum_{x,y,a}I_{x,y,a} \leq n^2h/2 - c\sqrt{h}*n + 1 }\\
&> \dfrac{1}{2} - 2c
\end{align}
Where the last inequality follows from Lemma \ref{lemma-anti-concentration} since each channel is sampled from a uniform distribution over a set of $2^b$ evenly spaced values between $0$ and $1$. Noting that $|Z| \leq 2^{(n^2hb)-1}$ since it cannot be larger than its complement yields $\frac{|Z'|}{|Z|} \geq 1-4c$. Therefore, $Z$ is $(1-4c)$-robust to $\elp{1}$-perturbations of size $c\sqrt{h}*n - 2$.\qed
\end{proof}

\begin{lemma}
\label{lemma-robust-l2}
$Z$ is $(1-4c)$-robust to $\elp{0}$-perturbations of size $c\sqrt{h}*n - 2$
\end{lemma}

\begin{proof}
It suffices to show that an image that is robust to $\elp{1}$-perturbations of size $d$ is also robust to $\elp{0}$-perturbations of size $d$, since the statement then follows directly from Lemma \ref{lemma-robust-l1}.

Let $I$ be an image that is not robust to $\elp{0}$-perturbations of size $d$, so there exists some $I'$ in a different class such that $\|I-I'\|_0 \leq d$. Then:
\begin{align}
    d &\geq \|I-I'\|_0\\
    &= \sum_{(x,y,a)}\cil{|I_{x,y,a} - I'_{x,y,a}|}\\ &\geq \sum_{(x,y,a)}|I_{x,y,a}-I'_{x,y,a}|\\
    &= \|I-I'\|_1
\end{align}
Where the second and third relations hold since channel values lie in $[0,1]$.

This implies that $I$ is not robust to $\elp{1}$-perturbations of size $d$. Therefore any image that is not robust to $\elp{0}$-perturbations of size $d$ is also not robust to $\elp{1}$-perturbations of size $d$. The contraposition yields the desired statement.\qed
\end{proof}

\begin{lemma}
$Z$ is $(1-4c)$-robust to $\elp{p}$-perturbations of size $\frac{(c\sqrt{h}*n - 2)^{1/p}}{2^b-1}$ for $p \geq 2$.
\end{lemma}

\begin{proof}
It suffices to show that any image that is robust to $\elp{0}$-perturbations of size $d$ is also robust to $\elp{p}$-perturbations of size $\frac{d^{1/p}}{2^b-1}$ for any $p \geq 2$, since the statement then follows directly from Lemma \ref{lemma-robust-l1}.

Let $I$ be an image that is robust to $\elp{0}$-perturbations of size $d$. Let $I'$ be any image in a different class, so $\|I-I'\|_0 > d$. Then for any $p \geq 1$:
\begin{align}
    \|I-I'\|_p^p &=
    \sum_{(x,y,a)}|I_{x,y,a}-I'_{x,y,a}|^p\\
    &\geq
    \sum_{(x,y,a)}  \dfrac{\cil{|I_{x,y,a} - I'_{x,y,a}|}}{(2^b-1)^p}\\
    &= \dfrac{\|I-I'\|_0}{(2^b-1)^p}\\
    &> \dfrac{d}{(2^b-1)^p}
\end{align}

Where the second relation follows from the fact that if two channel values differ, they must differ by at least $\frac{1}{2^b-1}$.

Therefore, $\|I-I'\|_p > \frac{d^{1/p}}{2^b-1}$ for any $I'$ whose class is different from $I$, so $I$ is robust to $\elp{p}$-perturbations of size $\frac{d^{1/p}}{2^b-1}$ for $p \geq 2$.\qed
\end{proof}
\subsection{Proof of Theorem \ref{thm-cont-discr}}
\label{appdx-cont-discr-proof}

Let $\mathcal{C}:\imgspace{n}{h}{b}\rightarrow \mathcal{Y}$ be any classifier, and let $C$ be any interesting class induced by $\mathcal{C}$. Our objective is to show that $C$ is not robust to various perturbations.

Let $T = \{[x*2^{-b}, (x+1)*2^{-b}) | x \in \mathbb{Z} \cap [0, 2^b-2] \} \cup \{[1-2^{-b}, 1]\}$ be a set of $2^b$ equal length intervals whose union is the interval $[0,1]$. Let $\discreizedh{n^2h}{2^b} = T^{n^2h}$ be their Cartesian power. Then the elements of $\discreizedh{n^2h}{2^b}$ are disjoint, and their union is precisely the hypercube $[0,1]^{n^2h}$.

We can associate each element of $\imgspace{n}{h}{b}$ with an element of $\discreizedh{n^2h}{2^b}$ by first mapping $\imgspace{n}{h}{b}$ to $[0,1]^{n^2h}$, which can be done by flattening the image tensor (which we denote by $\flat(I)$ for an image $I \in \imgspace{n}{h}{b}$). We then map that point to the element of $\discreizedh{n^2h}{2^b}$ the point falls within. The overall mapping is bijective, and we will denote it by $F$.

Let $\mathcal{A}:[0,1]^{n^2h} \times \mathbb{R} \rightarrow [0,1]^{n^2h} \cup \{\bot \}$ be a partial function that maps a point $p_1$ and a real value $c$ to a point $p_2$ such that the following hold:
\begin{enumerate}
    \item $\|p_1-p_2\|_2 \leq c$.
    \item Let $I_1, I_2 \in \imgspace{n}{h}{b}$ such that $p_1 \in F(I_1)$ and $p_2 \in F(I_2)$. Then we require that $\mathcal{C}(I_1) \neq \mathcal{C}(I_2)$.
\end{enumerate}

$\mathcal{A}(.)$ returns $\bot$ if and only if no such $p_2$ exists.

We can then define a procedure $\pertalg$ for finding a perturbation given an image $I$, which is outlined in Algorithm \ref{alg-perturb}.

\begin{algorithm}
\SetKwInOut{Input}{Input}
\Input{An image $I \in \imgspace{n}{h}{b}$ and a real values $c$.}
\KwResult{An image $I' \in \imgspace{n}{h}{b}$ such that $\mathcal{C}(I) \neq \mathcal{C}(I')$, or $\bot$.}
Sample $p_1$ from $F(I)$ uniformly at random\;
$p_2 \gets \mathcal{A}(p_1, c)$

\uIf{
$p_2 = \bot$
}{
\Return $\bot$\;
}
\uElse{
Find $I_2$ such that $p_2 \in F(I_2)$\;
\Return $I_2$\;
}
\caption{Find Perturbation}
\label{alg-perturb}
\end{algorithm}

Our proof strategy is to show that the perturbations found by $\pertalg$ are guaranteed to be small, and that the probability of failure is low. This must then imply that most images are not robust.

\begin{lemma}
\label{lemma-discc-len}
If $I' = \pertalg(I, c)$ is not $\bot$, then $\|I-I'\|_2 \leq c + 2\frac{n\sqrt{h}}{2^{b}}$.
\end{lemma}

\begin{proof}
Each element of $\discreizedh{n^2h}{2^b}$ has a diameter of $\frac{\sqrt{n^2h}}{2^b}$, thus $p_1$ differs from $\flat(I)$ be at most that distance. Similarly, $p_2$ differs from $\flat(I_2) = \flat(I')$ by that distance. We also must have $\|p_1-p_2\|_2 \leq c$ since $I' \neq \bot$. Putting it altogether with the triangle inequality we get $\|\flat(I) - \flat(I')\|_2 \leq c + 2\frac{n\sqrt{h}}{2^b}$. Since $\flat(.)$ preserves distances, we get the desired statement.\qed
\end{proof}

\begin{lemma}
\label{lemma-discc-prob}
If $I$ is drawn uniformly from $C$, then $\prob{\pertalg(I, c) = \bot} < 2e^{-c^2/2}$.
\end{lemma}

\begin{proof}
Let $F(C)$ denote the image of $C$ under $F$. Let $\bigcup F(C)$ denote the union of all elements in $F(C)$.

If the input $I$ is drawn uniformly from $C$, then $p_1$ is distributed uniformly over $\bigcup F(C)$. The procedure fails if and only if $\mathcal{A}(p_1, c) = \bot$, which happens if and only if all elements within a radius of $c$ from $p_1$ all belong to $\bigcup F(C)$. Let $C'$ denote the set of all such points.

\begin{align}
    \prob{\mathcal{A}(I_2,c) = \bot} &=
    \dfrac{ \mu (C') }{\mu( \bigcup F(C) )}\\
    &< 2e^{-c^2/2}
\end{align}

Where $\mu(.)$ denotes the Lebesgue measure.

The last inequality comes from Theorem \ref{thm-continuous-weak}, which is given in the next section. The statement applies for any set $S$ formed from a union of elements of $\discreizedh{n^2h}{2^b}$ whose measure is no larger than $1/2$. $\bigcup F(C)$ satisfies these criteria since $C$ is an interesting class, so we attain the desired statement.\qed
\end{proof}

\begin{lemma}
\label{lemma-discrization-result}
$C$ is not $2e^{-c^2/2}$-robust to $\elp{2}$-perturbations of size $c + 2\frac{n\sqrt{h}}{2^{b}}$.
\end{lemma}

\begin{proof}
Let $I$ be drawn uniformly from $C$. Let $C_r$ be the set of images that are robust to $\elp{2}$-perturbations of size $c + 2\frac{n\sqrt{h}}{2^{b}}$.

Let $I'=\pertalg(I,c)$. Then $I'$ is randomly distributed over $\imgspace{n}{h}{b}\cup\{ \bot \}$. By Lemma \ref{lemma-discc-len}, if $I' \in \imgspace{n}{h}{b}$, then $\|I-I'\|_2 \leq c + 2\frac{n\sqrt{h}}{2^{b}}$, which implies that $I \notin C_r$. By contraposition, $I \in C_r$ implies that $\pertalg(I,c) = \bot$. Therefore:

\begin{align}
    \prob{I' = \bot} &= \prob{ I \in C_r } + \prob{ I \notin C_r, I' = \bot }\\
    &\geq \prob{I \in C_r}\\
    &= \frac{|C_r|}{|C|}
\end{align}

By Lemma \ref{lemma-discc-prob}, $\prob{I' = \bot} < 2e^{-c^2/2}$. Thus, $\frac{|C_r|}{|C|} < 2e^{-c^2/2}$, which yields the desired statement.\qed
\end{proof}

\begin{lemma}
\label{lemma-thm3-higherp}
$C$ is not $2e^{-c^2/2}$-robust to $\elp{p}$-perturbations of size $\big( c + 2\frac{n\sqrt{h}}{2^{b}} \big)^{2/p}$ for $p \geq 2$.
\end{lemma}

\begin{proof}
We use the identical argument from Lemma \ref{lemma-thm1-higherp}.

Let $S_1$ be the set of images that are $r$-robust to $\elp{2}$-perturbations of size $d$, and let $S_2$ be the set of images that are $r$-robust to $\elp{p}$-perturbations of size $d^{2/p}$, where $p \geq 2$.

Suppose $I \notin S_1$. Then there exists some image $I'$ in a different class from $I$ such that $\|I-I'\|_{2} \leq d$. Therefore, for all $p > 0$, we have:

\begin{align}
    d^2 &\geq \|I-I'\|_{2}^2\\
    &= \sum_{x,y,c} |I_{x,y,c} - I'_{x,y,c}|^2\\
    &\geq \sum_{x,y,c} |I_{x,y,c} - I'_{x,y,c}|^{p}\\
    &= (\|I-I'\|_{p})^{p}
\end{align}

Where the third relation follows from the fact that channel values are contained in $[0,1]$. Therefore, $I \notin S_2$ either since $\|I-I'\|_{p} \leq d^{2/p}$. Taking the contraposition yields $S_2 \subseteq S_1$.

Setting $d = c + 2\frac{n\sqrt{h}}{2^{b}}$ and applying Lemma \ref{lemma-discrization-result} gives the desired result.\qed
\end{proof}
\subsection{Proof of Theorem \ref{thm-continuous-weak}}
\label{appdx-continuous-results}
Our objective in this section is to complete the proof of Theorem \ref{thm-cont-discr} by proving Theorem \ref{thm-continuous-weak}, stated below. We will use $\mu(.)$ to denote Lebesgue measure throughout this section.

\begin{definition}
We say a set $S \subseteq [0,1]^n$ is a regular set if there is some $q$ and $T \subseteq \discreizedh{n}{q}$ such that $S = \bigcup_{t \in T}t$.
\end{definition}

\begin{theorem}
\label{thm-continuous-weak}
Let $S \subseteq [0,1]^n$ be a regular set such that $\mu(S) \leq 1/2$. Let $S_r \subseteq S$ contain all the points in $S$ such that for all $y \in [0,1]$, $\|x-y\|_2 \leq r \implies y \in S$. Then $\frac{\mu(S_r)}{\mu(S)} < 2e^{c^2/2}$.
\end{theorem}

\subsubsection{Properties of the standard normal distribution}

First, we define the cumulative distribution function for the standard normal distribution and its derivative.

\begin{align}
    \normcdf(x) &= \int_{-\infty}^{x} \dfrac{1}{\sqrt{2\pi}}e^{-t^2/2} dt\\
    \normcdf'(x) &= \dfrac{1}{\sqrt{2\pi}}e^{-x^2/2}
\end{align}

Similarly to the discrete case, the ratio of the cumulative distribution functions is monotonic increasing.

\begin{lemma}
\label{lemma-cont-monotonic}
$\frac{\normcdf(x-k)}{\normcdf(x)}$ is monotonic increasing in $x$ for all $k \geq 0$.
\end{lemma}

\begin{proof}
Let $f(x) = \frac{e^{-x^2/2}}{\int_{-\infty}^{x} e^{-t^2/2}dt}$. Then:

\begin{align}
    \dfrac{d}{dx}f(x) = \dfrac{
    -e^{-x^2/2}x\int_{-\infty}^{x} e^{-t^2/2}dt
    -e^{-x^2/2}e^{-x^2/2}
    }
    {\big(
    \int_{-\infty}^{x} e^{t^2/2} dt
    \big)^2}
\end{align}

When $x \geq 0$, this derivative is negative since both terms in the numerator are negative. If $x < 0$, we have the following:

\begin{align}
    -x\int_{-\infty}^{x} e^{-t^2/2}dt
    &<
    -x\int_{-\infty}^{x} e^{-t^2/2} + \dfrac{1}{t^2}e^{-t^2/2}dt\\
    &=
    -x\big(
    -\dfrac{1}{t}e^{-t^2/2} \bigg\rvert_{-\infty}^{x}
    \big)\\
    &=
    e^{-x^2/2}
\end{align}

So the sum is strictly smaller than $(e^{-x^2/2})^2-(e^{-x^2/2})^2 = 0$. Therefore, the derivative is everywhere negative, so $f(x)$ is strictly decreasing.

Therefore, we have the following for any non-negative $k$:

\begin{align}
    \dfrac{d}{dx}ln(\dfrac{\normcdf(x-k)}{\normcdf(x)})
    = f(x-k) - f(x) \geq 0
\end{align}

Since $ln(.)$ is a monotonic increasing function, $\frac{\normcdf(x-k)}{\normcdf(x)}$ must also be monotonic increasing.\qed
\end{proof}

\subsubsection{Proving Theorem \ref{thm-continuous-weak}}

Similarly to the discrete case, our main result relies on an isoperimetry statement, this time on the unit hypercube~\citep{barthe2000some}.

\begin{lemma}[Isoperimetric Theorem on the Unit Hypercube]
\label{lemma-cont-isoperim}
For any $n$, let $A \subset [0,1]^n$ be a Borel set. Let $A_{\epsilon} = \{x \in [0,1]^n \big| \exists x' \in A: \|x-x'\| \leq \epsilon \}$. Then we have the following:
\begin{equation}
    \liminf_{\epsilon \rightarrow 0^+} \dfrac{\mu(A_{\epsilon}) - \mu(A)}{\epsilon} \geq \sqrt{2\pi} \normcdf'(\normcdf^{-1}(\mu(A)))
\end{equation}
\end{lemma}

Let $C \subseteq [0,1]$ be a regular set such that $0 < \mu(C) \leq 1/2$. Let $C_r \subseteq C$ denote the points $p_1$ in $C$ such that for any point $p_2 \in [0,1]$, $\|p_1-p_2\|_2 \leq r \implies p_2 \in C$.

\begin{lemma}
$C_r \leq \normcdf( \normcdf^{-1}(\mu(C))-r )$
\label{lemma-cont-tailbound}
\end{lemma}

\begin{proof}
Let $z=\normcdf^{-1}(\mu(C))$ and let $f(x) = \normcdf( x+z )$. Let $v(.)$ be a Lebesgue integrable function such that the following holds:

\begin{align}
    V(r) = \int_{(-\infty,r)} v(t)dt &= \begin{cases}
\mu(C_{-r}) & \text{if $r \leq 0$}\\
\mu(C_0) & \text{otherwise}
\end{cases}
\end{align}

This exists since $C$ is a regular set. Since $V(x)$ results from integration, it is also a continuous function.

It then suffices to show that $V(x) \leq f(x)$ for all $x$, since $V(x)$ corresponds to the left hand side of the theorem statement and $f(x)$ corresponds to the right hand side. Suppose this is not the case. We know that $V(x) \leq f(x)$ for all $x \geq 0$, so if this is violated it must happen when $x < 0$. Since $V(x)$ and $f(x)$ are both continuous, by the intermediate value theorem there must exist some interval $[a,b)$ where $V(x) > f(x)$ if $x \in [a,b)$, $V(b)=f(b)$, and $a < b \leq 0$.

This gives us the following:

\begin{align}
    V(b)-V(a) &= \int_{[a,b)}v(t)dt\\
    \label{eqn-helper-20}
    &=
    \int_{[a,b)\setminus Z}\lim_{\epsilon \rightarrow 0^+}
    \dfrac{V(t+\epsilon)-V(t)}{\epsilon}dt\\
    &=
    \int_{[a,b)\setminus Z}\liminf_{\epsilon \rightarrow 0^+}
    \dfrac{\mu(C_{-t-\epsilon})-\mu(C_{-t})}{\epsilon}dt\\
    \label{eqn-helper-22}
    &\geq
    \int_{[a,b)} \sqrt{2\pi} \normcdf'(\normcdf^{-1}(\mu(C_{-t}))) dt\\
    \label{eqn-helper-23}
    &\geq
    \int_{[a,b)} \sqrt{2\pi} \normcdf'(\normcdf^{-1}(f(t))) dt\\
    &\geq
    f(b)-f(a)
\end{align}

Where $Z$ is the set of values where the limit in Equation \ref{eqn-helper-20} is not equal to $v(t)$, which by the Lebesgue differentiation theorem is a set of measure 0. Equation \ref{eqn-helper-22} is an application of Lemma \ref{lemma-cont-isoperim}, which is applicable since $C_{-t}$ is a Borel set due to the $C$ being a regular set. Equation \ref{eqn-helper-23} follows from the fact that $f(x) \leq V(x)$ for all $x \in [a,b]$ and the fact that $\normcdf'(\normcdf^{-1}(.))$ is monotonically increasing if the input is no greater than $1/2$.

We also have $V(a) > f(a)$ and $V(b) = f(b)$, so it must be the case that $V(b)-V(a) < f(b)-f(a)$. This contradicts the above, so it must be the case that $V(x) \leq f(x)$ for all $x$.\qed
\end{proof}

\begin{lemma}
$\mu(C_c) < 2e^{-c^2/2}\mu(C)$
\label{lemma-cont-main}
\end{lemma}

\begin{proof}
Let $z = \normcdf^{-1}(\mu(C))$. Then for any $c \geq 0$,
\begin{equation}
    \dfrac{\mu(C_c)}{\mu(C)} \leq \dfrac{\normcdf(z-c)}{\normcdf(z)} \leq \dfrac{\normcdf(1/2 - c)}{\normcdf(1/2)}
    < 2e^{-c^2/2}
\end{equation}

Where the first inequality follows from Lemma \ref{lemma-cont-tailbound}, the second inequality follows from Lemma \ref{lemma-cont-monotonic} and the fact that $\mu(C) \leq 1/2$, and the third inequality follows from the Gaussian tail bound $\normcdf(x) < e^{-x^2/2}$ for all $x \leq 1/2$.\qed
\end{proof}
\subsection{Average distance between images}
\label{appdx-avg-distance}

We wish to show that for a pair of images $I, I' \in \imgspace{n}{h}{b}$ that are sampled independently and uniformly, there exists a $k_{h,b,p}$ such that:

\begin{align}
    \mathbb{E}[\|I-I'\|_p] \geq k_{h,b,p}n^{2/\max(1,p)}
\end{align}

First, we note that we have:

\begin{align}
    \mathbb{E}[\|I-I'\|_p^{\max(1,p)}]
    &= n^2h * \mathbb{E}[ |X-Y|^{\max(1,p)} ]
\end{align}

Where $X$ and $Y$ are independent random variables that are both drawn uniformly from a set of $2^b$ equally spaced values, where the largest is 1 and the smallest is 0. For simplicity, we denote $\mathbb{E}[ |X-Y|^{\max(1,p)} ]$ with $k_{b,p}$.

$\|I-I'\|_p^{\max(1,p)}$ is non-negative and cannot be larger than $n^2h$. Therefore, the probability that $\|I-I'\|_p^{\max(1,p)} \geq n^2hk_{b,p}/2$ is at least $\frac{k_{b,p}}{2-k_{b,p}}$.

Via a monotonicity argument we can deduce that the probability that $\|I-I'\|_p \geq (hk_{b,p}/2)^{1/\max(p,1)}n^{2/\max(p,1)}$ is at least $\frac{k_{b,p}}{2-k_{b,p}}$ as well. We can then apply Markov's inequality to get the following:

\begin{align}
    \mathbb{E}[\|I-I'\|_p] \geq \frac{k_{b,p}}{2-k_{b,p}}(hk_{b,p}/2)^{1/\max(p,1)}n^{2/\max(p,1)}
\end{align}

By setting $k_{h,b,p}$ to be $\frac{k_{b,p}}{2-k_{b,p}}(hk_{b,p}/2)^{1/\max(p,1)}$ we attain our desired result.
\end{document}